\documentclass[10pt,twocolumn,letterpaper]{article}

\usepackage[pagenumbers]{wacv_template/wacv} %

\definecolor{wacvblue}{rgb}{0.21,0.49,0.74}
\usepackage[pagebackref,breaklinks,colorlinks,allcolors=wacvblue]{hyperref}

\def\wacvPaperID{2605} %
\def\confName{WACV}
\def\confYear{2026}

\definecolor{bittersweet}{rgb}{1.0, 0.44, 0.37}
\definecolor{mypurple}{rgb}{0.5, 0, 0.5}
\newcommand{\red}[1]{{\color{red}#1}}
\newcommand{\orange}[1]{{\color{bittersweet}#1}}

\newcommand{\blue}[1]{{\color{cyan}#1}}
\newcommand{\purple}[1]{{\color{mypurple}#1}}

\newcommand{\xhdr}[1]{\vspace{5pt}\noindent\textbf{#1}}

\usepackage{graphicx}
\usepackage{amsthm}   %
\usepackage{amsmath, amssymb} %
\usepackage{tabularx}
\usepackage{multirow}
\usepackage{float}
\usepackage{arydshln}

\usepackage{enumitem}

\usepackage{pifont}%

\usepackage{comment}

\newcommand{\cmark}{\ding{51}}%

\newtheorem{theorem}{Theorem}[section]

\newenvironment{application}[1][Application]{%
  \par\pushQED{\qed}\normalfont
  \topsep6pt plus6pt minus3pt
  \trivlist
  \item[\hskip\labelsep\bfseries #1.]%
}{\popQED\endtrivlist\par}

\usepackage{hyperref} %
\usepackage{tikz}
\usetikzlibrary{spy}
\usetikzlibrary{matrix}

\usepackage[symbol]{footmisc}

\def\projecturl{https://www.ayshrv.com/defocus-blur-gen}

\author{
Ayush Shrivastava\textsuperscript{1,2$^{*}$}\qquad
Connelly Barnes\textsuperscript{2}\qquad
Xuaner Zhang\textsuperscript{2}\qquad
Lingzhi Zhang\textsuperscript{2}\vspace{1.2mm}\\
Andrew Owens\textsuperscript{1}\qquad
Sohrab Amirghodsi\textsuperscript{2}\qquad
Eli Shechtman\textsuperscript{2}
\vspace{2mm}\\
\textsuperscript{1}University of Michigan \qquad \textsuperscript{2}Adobe Research
\vspace{1mm}\\
{\normalsize \texttt{\url{\projecturl}}}
}

\begin{document}

\title{Fine-grained Defocus Blur Control for Generative Image Models}

\maketitle

{\let\thefootnote\relax\footnotetext{{* Work done during an internship at Adobe.}}}

\begin{abstract}
Current text-to-image diffusion models excel at generating diverse, high-quality images, yet they struggle to incorporate fine-grained camera metadata such as precise aperture settings. In this work, we introduce a novel text-to-image diffusion framework that leverages camera metadata, or EXIF data, which is often embedded in image files, with an emphasis on generating controllable lens blur. Our method mimics the physical image formation process by first generating an all-in-focus image, estimating its monocular depth, predicting a plausible focus distance with a novel focus distance transformer, and then forming a defocused image with an existing differentiable lens blur model~\cite{wang2023neural}. Gradients flow backwards through this whole process, allowing us to learn without explicit supervision to generate defocus effects based on content elements and the provided EXIF data. At inference time, this enables precise interactive user control over defocus effects while preserving scene contents, which is not achievable with existing diffusion models. Experimental results demonstrate that our model enables superior fine-grained control without altering the depicted scene.
\end{abstract}
\vspace{-5mm}
    
\section{Introduction}

\begin{figure}[t]
    \centering
    \includegraphics[width=3.4in]{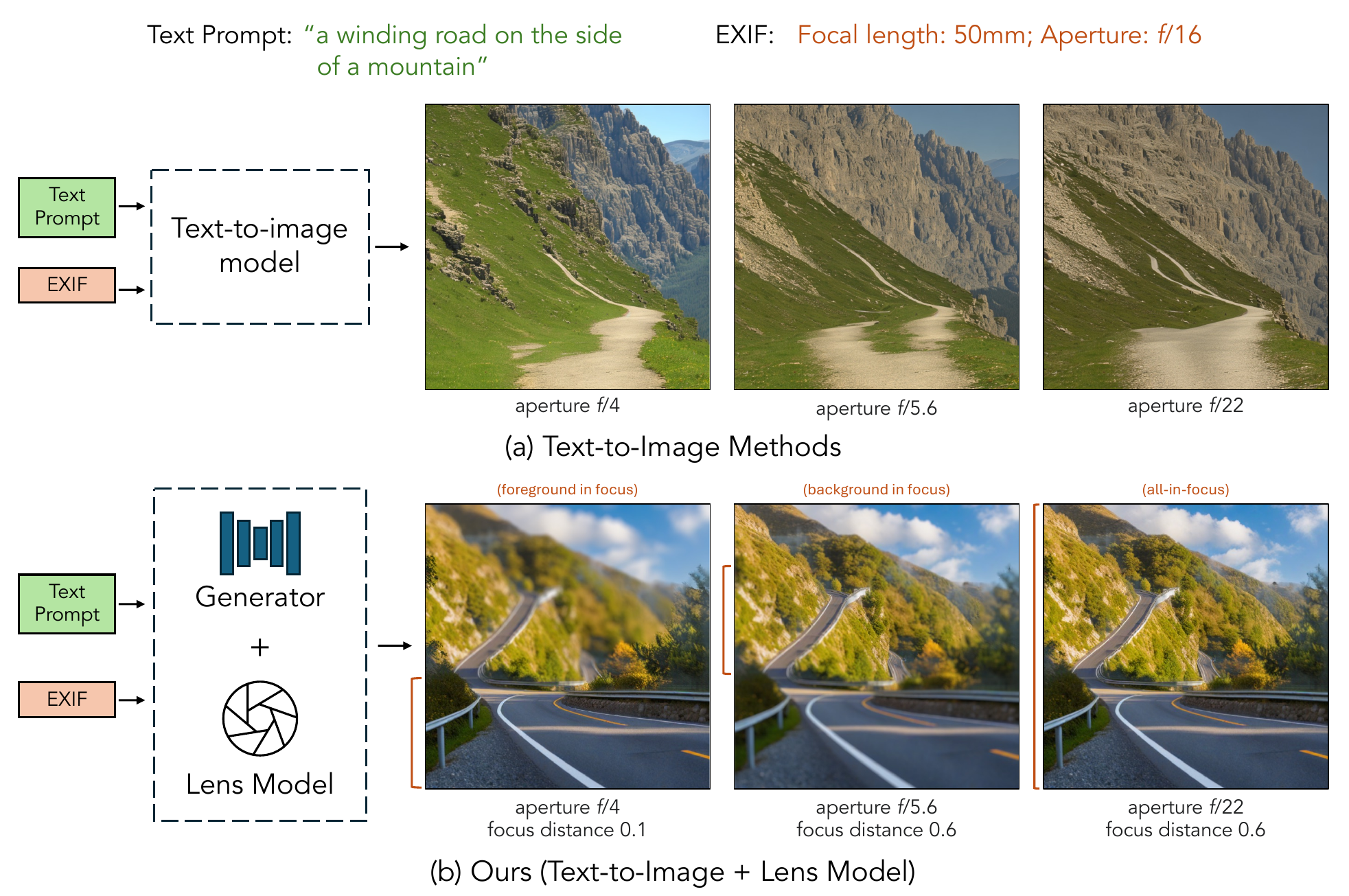}\vspace{-2mm}%
    \caption{{\bf Overview.} We propose a text-to-image model that precisely controls the amount and location of defocus blur in generated images while preserving the scene content. (a) Previous methods like Fang~et~al.~\cite{fang2024} (shown here), struggle to preserve scene content when modifying aperture settings. (b) Our model generates images with varying defocus blur while keeping the scene intact. Each image exhibits a distinct blur effect based on the specified aperture and focus distance. At lower focus distances, nearby objects appear sharper (foreground in focus), whereas at higher focus distances, the focus shifts to distant objects (background in focus). The orange bar \orange{[} on the left of each image indicates the in-focus region. Higher f-stops result in reduced blur, producing an all-in-focus image for $f/22$. 
    }
    \label{fig:teaser}
    \vspace{-8mm}
\end{figure}

Recent advances in text-to-image generative models have shown impressive capabilities in producing realistic, high-quality images~\cite{podell2023sdxl}. However, in a real photo shoot, a photographer using a DSLR or mirrorless camera might adjust defocus effects (modifying the focal plane and aperture) to direct attention to key subjects and de-emphasize background details. Current generative models lack the ability to make such fine-grained adjustments while preserving scene integrity, as would occur in an actual shoot.

Many real-world images encode details such as focal distance, aperture, lens type, and camera model in EXIF (EXchangeable Image File Format) metadata. A generative model that can \emph{conditionally} and \emph{controllably} adapt to these EXIF parameters would be highly valuable. For example, given a prompt and an aperture value of $f/1.8$, the model should generate an image that accurately reflects both the prompt and the specified aperture. If the user wishes to reduce blur, they could specify a smaller aperture (larger f-stop), prompting the model to generate the same scene with less blur. This task is inherently challenging, as the model must determine the appropriate amount of blur to add or remove, identify which regions should be blurred by recognizing salient areas, and apply blur selectively to create an aesthetically pleasing image while preserving scene content. Achieving a balance between content preservation and depth-aware spatial blur control makes this a complex but crucial problem in generative image modeling.

One of the major challenges is disentangling the image content from the camera parameters.
For example, recent work~\cite{fang2024} has attempted to address this by training camera parameter embeddings for text-to-image models to guide generation conditioned on EXIF tags. However, that approach struggles to maintain scene consistency when adding or reducing defocus blur across different camera settings.  Likewise, camera properties and scene content are closely intertwined in off-the-shelf text-to-image models. As a result, applying camera effects as a postprocessing step is ineffective, since the generated images already contain other camera properties (e.g., existing defocus blur). %

We propose a method that decouples scene generation from lens properties, enabling control over EXIF-based image generation while preserving the physicality of image formation. Our model produces plausible focal distances and grants users fine-grained control over lens parameters while maintaining scene content. An example of this control is shown in Figure~\ref{fig:teaser}.

Our approach is based on training a few-step generative model conditioned on EXIF information (specifically aperture). This model obtains a supervision signal from a differentiable lens model.
We train a fast few-step generative model conditioned on EXIF information to generate deep (or shallow) depth-of-field (DoF) images using weak supervision from simply extracting high-quality deep (and shallow) DoF subsets from the dataset. In particular, the shallow DoF images are generated by our novel focus distance transformer, which predicts a focus distance and scale used by a differentiable lens blur model based on Wang~et~al.~\cite{wang2023neural}. Because focus distances are typically omitted or unreliable in EXIF tags, we show that a weakly supervised approach can be used to learn these directly from shallow DoF images. %
By explicitly modeling defocus blur within the lens, our approach enables end-to-end learning of plausible focal distances while manipulating lens properties separately from scene content.
We base our generator on improved Distribution Matching Distillation (DMD2)~\cite{yin2024improved} since it can simultaneously learn a fast few-step model to enable user interactivity, and its unpaired DMD and GAN losses are suited for our weak supervision. %
While these components have been studied in isolation, our key contribution lies in integrating them into a unified, unsupervised training framework that enables precise and intuitive user control over the defocus blur effect --- empowering interactive generation of images with controllable DoF.

Our main contribution is a generative framework that enables learning lens properties and allowing precise interactive user control over such lens properties while preserving scene content. Our experiments show clear improvements, enabling such fine-grained user interactivity. %
Our technical contributions are:
(1) A unified generative pipeline that learns without explicit supervision to model depth of field, which is done with the help of %
(2) A novel focus distance transformer that learns to predict plausible focus distances and scales, which are used in a physically-inspired thin lens blur model~\cite{wang2023neural}.
(3) We show that weak supervision from shallow and deep depth-of-field images can successfully supervise a diffusion model with a lens model, and that this form of supervision can easily be obtained from unlabeled image datasets.

\section{Related Work}

\paragraph{Depth of field rendering.} 
Finite apertures produce defocus effects long studied for light-field rendering~\cite{adams2007general} and adaptive ray-tracing sampling~\cite{belcour20135d,chen2011efficient,zwicker2015recent}, with recent differentiable depth-of-field-aware rendering~\cite{pidhorskyi2022depth}.
AR-GAN~\cite{kaneko2021unsupervised} estimates a depth map and all-in-focus RGB image in an unsupervised manner and integrates over a finite aperture to render defocus; unlike it, we adopt a diffusion framework for text-to-image synthesis with learned modules such as the focal-distance model.
AR-NeRF~\cite{kaneko2022ar} uses a similar unsupervised approach based on NeRFs~\cite{mildenhall2021nerf}, while DOF-GS~\cite{wang2024dof} extends 3D Gaussian splatting~\cite{kerbl20233d} with a thin-lens circle-of-confusion for 3D reconstruction and defocus.
Wang~et~al.~\cite{wang2023neural} reconstruct an all-in-focus HDR radiance map and depth map from time-aperture-focus stacks, whose differentiable thin lens model we use, while Xin~et~al.~\cite{xin2021defocus} recover all-in-focus images and defocus maps from dual-pixel images via multiplane optimization.
Dr.Bokeh~\cite{sheng2024dr} uses a layered scene representation with a differentiable occlusion-aware bokeh model, which we adopt as a plug-and-play lens component during inference.
DC$^2$~\cite{alzayer2023defocuscontrol} enables post-capture defocus control by fusing dual-camera inputs with different fixed apertures but requires multi-camera hardware.
DiffCamera~\cite{wang2025diffcamera} performs arbitrary post-capture refocusing using a diffusion transformer trained on simulated multi-focus data, whereas our method learns defocus from in-the-wild data and provides fine-grained control over focus and aperture.

\vspace{-4mm}
\paragraph{Diffusion models.} We use SDXL~\cite{podell2023sdxl}, a text-to-image latent diffusion model~\cite{Rombach_2022_CVPR} as our generator architecture. We estimate depth with frozen diffusion-based monocular metric depth estimation model Metric3Dv2~\cite{hu2024metric3d}.
Voynov~et~al.~\cite{voynov2024curved} uses a per-pixel coordinate conditioning method to generate diffusion images using different optical systems, such as fisheye and concave lenses, and spherical panoramas. 
Diffusion-based image restoration methods such as SUPIR~\cite{yu2024scaling} have shown strong performance on camera lens-related restoration tasks like recovering from mixed blur, super-resolution, and/or noise degradations.

\vspace{-4mm}
\paragraph{Learning with metadata.} 
Several models leverage metadata for conditioning. DiffusionSat~\cite{khanna2023diffusionsat} generates satellite imagery conditioned on metadata such as GPS, date, cloud cover, some of which also appear in EXIF metadata for cameras. EXIF as Language~\cite{zheng2023exif} uses contrastive learning to connect image patches with EXIF data for spliced image detection. Camera Settings as Tokens~\cite{fang2024} trains a LoRA~\cite{hu2021lora} adapter and camera parameter embeddings to condition generation on focal length, aperture, ISO, and exposure.
but struggles to preserve scene content and consistently control blur. In contrast, our method decouples scene generation from defocus blur for precise control. Generative Photography~\cite{yuan2024generative}, learns camera embeddings using paired videos with varying EXIF settings, whereas our method is entirely unsupervised, leveraging only unpaired deep and shallow depth-of-field data. A concurrent work, Bokeh Diffusion~\cite{fortes2025bokeh}, conditions a diffusion model on a defocus parameter for scene-consistent bokeh control but relies on synthetic blur for supervision.

\vspace{-4mm}
\paragraph{Few-step diffusion model distillation.} 
Teacher models can be distilled to few-step students via progressive distillation~\cite{salimans2022progressive,meng2023distillation}, GAN-based approaches (ADD~\cite{sauer2025adversarial}, LADD~\cite{sauer2024fast}, Diffusion2GAN~\cite{kang2024distilling}), and distribution-matching methods DMD~\cite{yin2024one} and DMD2~\cite{yin2024improved}.
We adopt DMD2 pipeline for efficiency and—crucially—to match the shallow and deep DoF data distributions (which are unpaired) within a single pipeline.

\begin{figure}[t]
    \centering
    \includegraphics[width=3.36in]{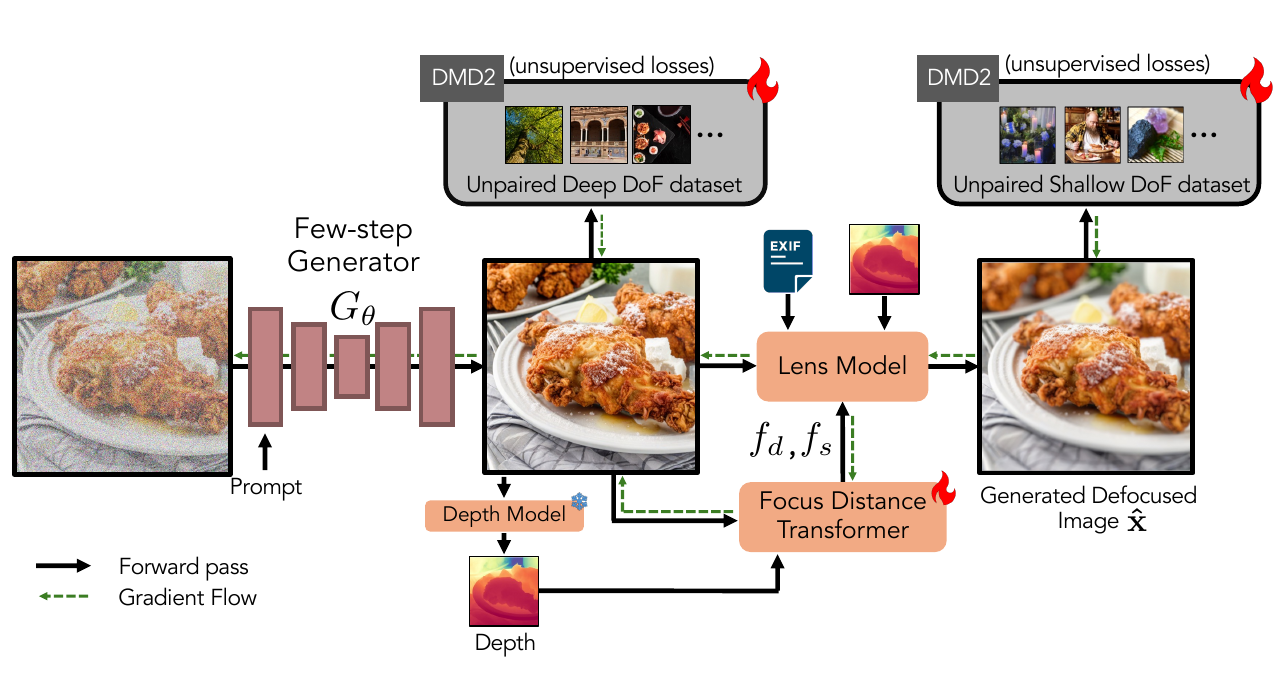}
    \vspace{-8mm}%
    \caption{{\bf Model Architecture.} We propose an image generation method that enables precise control over blur intensity and location within an image. This model obtains its supervision from a differentiable lens model and training examples of images with shallow and deep depth-of-field. We train our model to generate an \emph{all-in-focus} image using $G_{\theta}$. A depth model then predicts depth for this image, which, along with the image itself, is fed into a model that estimates the focus distance and depth scale.  Finally, a lens model combines EXIF data with these predictions to apply spatially varying blur, generating the final image. %
    We train the all-in-focus generator using unsupervised 
    DMD2~\cite{yin2024improved} losses on our unpaired Deep DoF dataset and optimize the entire pipeline with DMD2 losses on the unpaired Shallow DoF dataset.} 
    \label{fig:model_architecture}
    \vspace{-6mm}
\end{figure}

\section{Approach}
Our objective is to develop a text-to-image model conditioned on EXIF metadata, enabling precise, fine-grained control over the degree and location of image defocus blur. To achieve this, we decouple image generation from the blurring process by introducing a lens model that applies controlled blur based on depth information. Our pipeline is shown in Figure~\ref{fig:model_architecture} and begins with a fast, few-step generator that produces an all-in-focus image. This is followed by a focal distance model that identifies regions of interest and guides those areas to stay in focus. Finally, the lens model applies a spatially-varying blur according to a differentiable physically-inspired thin lens model. The entire pipeline is differentiable and learned end-to-end.%

\subsection{Generator}
Our generator builds on the few-step generator architecture from DMD2~\cite{yin2024improved}, which distills a Stable Diffusion-XL (SDXL) teacher model into a fast few-step generator by aligning the generator’s output distribution with that of the teacher model. The process involves two key components:

\xhdr{Distribution Matching Distillation (DMD).} This component distills the teacher diffusion model (SDXL) into a few-step generator $G$ by minimizing the Kullback-Liebler (KL) divergence between the diffused target distribution $p_{\text{real},t}$ and the diffused generator output distribution $p_{\text{fake},t}$ for a timestep $t$. The loss is derived from the difference between two score functions and is computed as:
\begin{equation}
    \small
    \begin{aligned}
    \nabla\mathcal{L}_\text{DMD} &=
        \mathbb{E}_t\left(\nabla_\theta \text{KL}(p_{\text{fake},t} \| p_{\text{real},t}) \right) \\ &= 
        \mathbb{E}_t\left(\int \big(s_{\text{fake}}(x_t,t) - s_{\text{real}}(x_t,t)\big) \frac{dG_\theta(z)}{d\theta} \hspace{.5mm} dz\right),
    \end{aligned}
    \label{eq:kl-grad}
\end{equation}
where $s_{\text{real}}$ and $s_{\text{fake}}$ are the score functions approximated using diffusion models $\mu_{\text{real}}$ and $\mu_{\text{fake}}$. Here, $t \sim \mathcal{U}[0, T]$, $x_t=F(G_\theta(z),t)$ where $F$ is the forward diffusion process (adding noise), $z\sim \mathcal{N}(0,\mathbf{I})$ is random Gaussian noise, $\theta$ denotes the generator parameters. $\mu_{\text{real}}$ is the pre-trained frozen diffusion model (SDXL) used as the teacher and $\mu_{\text{fake}}$ is a dynamically trained diffusion model that is optimized alongside $G$ using the denoising score-matching loss ($\mathcal{L}_\text{denoise}$), conditioned on the output of $G$.

\xhdr{GAN Loss.} A discriminator $D$ is trained to distinguish between real images and those generated by $G$. A classification branch is added to the bottleneck features of the fake diffusion denoiser $\mu_{\text{fake}}$, with the loss given by:
\begin{equation}
\label{eq:gan_loss}
\begin{aligned}
\mathcal{L}_{\text{GAN}} &= \mathbb{E}_{x \sim p_{\text{real}}}[\log D(F(x, t))] \\ &+ \mathbb{E}_{z \sim p_{\text{noise}}}[-\log(D(F(G_\theta(z), t)))],
\end{aligned}
\end{equation}
where $t \sim \mathcal{U}[0, T]$ and $D$ represents the discriminator.

\xhdr{Training.} The generator $G$ and $\mu_{\text{fake}}$ are initialized with a pre-trained diffusion model. During training, $G$ minimizes $\mathcal{L}_\text{DMD} + \mathcal{L}_{\text{GAN}}$, while $\mu_{\text{fake}}$ minimizes $\mathcal{L}_\text{denoise} + \mathcal{L}_{\text{GAN}}$. %
This setup can be summarized as DMD2~($G$, $\mu_{\text{fake}}$) where $G$ is a few-step generator conditioned on a text prompt $t$ and generates an all-in-focus image $\mathbf{x}$ from a noise sample $z$:

\begin{equation}
\mathbf{x} = G_{\theta}(z, t)
\end{equation}

\subsection{Focus Distance Transformer}
To achieve aesthetically pleasing blur without rendering images entirely out of focus, the model must determine the salient regions in the image and decide which areas should remain sharp. In photographic terms, this requires selecting a depth for the focal plane to ensure key objects appear in focus. To this end, we develop a focus distance prediction model that outputs the focus distance (in depth units) at which the most important objects are located. We first extract a monocular depth map $\mathbf{d}$ from the generated all-in-focus image using a frozen Metric3Dv2~\cite{hu2024metric3d} model. Next, the focus distance transformer takes as input the generated all-in-focus image and depth map, and produces the focus distance $f_d$. Additionally, it produces a scale factor $f_s$, which allows the generative model to align the metric depth and focal length and still generate plausible results when either is inaccurate. Our model fine-tunes the Visual Saliency Transformer (VST)~\cite{Liu_2021_ICCV}, which we adapt for this specific task.

\xhdr{Focus Distance.} To predict the focus distance, we compute a saliency map from our VST network decoder and take its weighted average with the depth map, yielding a focus distance prediction within the range of the depth map. 
\begin{equation}
f_d = \lVert \mathbf{d} \boldsymbol{\odot} \text{VST}(\mathbf{x}, \mathbf{d}) \rVert_1 / \lVert\text{VST}(\mathbf{x}, \mathbf{d}) \rVert_1
\end{equation}

To supervise learning the focal distance, we use a frozen pre-trained copy of the VST network to calculate a reference focus distance from this weighted average. We apply a weak supervision for $f_d$ with a Huber loss between the reference focus distance from the pretrained VST network.%

\xhdr{Focus Distance Scale.} For the focus distance scale, we extract the saliency token $\text{VST}_\text{SAL}$ 
 from the first layer of the VST decoder and use a linear head to predict its value.  This is learned through end-to-end training with DMD2 losses.
\begin{equation}
f_s = \text{MLP}(\text{VST}(\mathbf{x}, \mathbf{d})_{\text{SAL}})
\end{equation}

\subsection{Lens Model}
\label{sec:lens_model}

To achieve precise control over defocus blur within the image generation pipeline, we require a differentiable lens blur model that can be trained jointly with the generator, allowing gradients to flow back from the defocused image. For this, we use the thin lens model from ~\cite{wang2023neural}, which produces defocus blur in the image through differentiable kernels. The lens model is parameterized by focal length $f$, focus distance $f_d$ and aperture $N$. We extend this model to include the focus distance scale $f_s$ mentioned above. %
The circle of confusion (CoC) disk diameter is computed as:

\begin{equation}
\textbf{coc} = \frac{|\mathbf{d}-f_d|}{\mathbf{d}} \frac{f^2}{N(f_s \cdot f_d - f)}
\end{equation}

After computing the \textbf{coc} for each pixel, we simulate defocus blur using a spatially-varying convolution $\text{W}$ as $\mathbf{\hat{x}} = W \ast \mathbf{x}$ as implemented in ~\cite{wang2023neural}. 
The kernel $\text{W}$ is a unit-energy disk with \textbf{coc} as its diameter and a differentiable soft boundary based on pixel distance from the kernel center.
We denote the combination of the lens model and focus distance model as $\hat{G}$ such that $\hat{G}(\mathbf{x}) = \mathbf{\hat{x}}$. For convenience, we refer to this lens model as the \emph{TAF Lens}~\cite{wang2023neural}.

\xhdr{Swapping the Lens Model at Inference.}
Our framework allows swapping out the lens model at inference time. %
For example, we use Dr.Bokeh~\cite{sheng2024dr} as an alternative lens model to demonstrate results with our generator during inference. The TAF lens model is fast and therefore suitable for training, whereas the Dr.Bokeh lens model is slower but produces higher-quality results due to its layered representation and inpainting.%

\subsection{Deep and Shallow Depth-of-Field Datasets}
\label{sec:datasets}
We desire a generator $G$ that produces all-in-focus, or deep depth-of-field (DoF) images $\mathbf{x}$, and a lens model that generates shallow DoF images $\mathbf{\hat{x}}$. We train these models using only a weak form of supervision. We construct datasets of deep and shallow DoF images from a large pool of uncurated images. We provide details about dataset curation in Supplemental Section~\ref{sec:suppl_datasets}.

\subsection{Putting Everything Together}
In summary, our generator $G$ produces deep depth-of-field (DoF) images, $\mathbf{x}$, while the lens model $\hat{G}$ generates shallow DoF images, $\mathbf{\hat{x}}$. We aim for $\mathbf{x}$ to have a deep DoF and $\mathbf{\hat{x}}$ to have a shallow DoF. To achieve this, we train $G$ on the Deep DoF dataset using DMD2~($G$, $\mu_{\text{fake}}$) and train $\hat{G}$ on the Shallow DoF dataset using DMD2~($\hat{G}$, $\mu_{\text{fake}}$). %
Our overall loss is: 
\begin{equation}
\lambda_1\text{DMD}2(G, \mu_{\text{fake}}) + \lambda_2\text{DMD}2(\hat{G}, \mu_{\text{fake}}) + \lambda_3 L_{\text{Huber}}
\end{equation}

\vspace{1mm}
\section{Experiments}

We use a commercially available stock-photography dataset for training. We extract 1.5M samples for the deep DoF and shallow DoF datasets. We pretrain the generator $G$ on the deep DoF dataset. Then we jointly fine-tune $G$ on the deep DoF dataset and $\hat{G}$ on the shallow DoF dataset.

\begin{table*}[ht]
\small
\centering 
{
\setlength{\tabcolsep}{2.5pt}

\caption{{\bf Results.} We compare our method to teacher and distilled-SDXL models, variants of these with lens models, and other baselines.} %
\vspace{-2ex}
\begin{tabularx}{1\linewidth}{c c X l l cccc ccccc} 
\toprule
& \multirow{2}{*}{\textbf{\#}} & \multirow{2}{*}{\textbf{Method}} & \textbf{Lens} & \textbf{EXIF}   & \textbf{Blur} $\uparrow$   & \textbf{Content} $\uparrow$ & \multirow{2}{*}{\textbf{LPIPS} $\downarrow$}  & \multirow{2}{*}{$\textrm{\bf FID}_\textrm{\bf DDoF} \downarrow$} & \multirow{2}{*}{$\textrm{\bf FID}_\textrm{\bf SDoF}\downarrow$} \\
& & & \textbf{Model} & \textbf{Conditioning} & \textbf{Monotonicity}& \textbf{Consistency}  \\

\midrule
\parbox[t]{1mm}{\multirow{5}{*}{\rotatebox[origin=c]{90}{Gen.}}}
& 1 & SDXL~\cite{podell2023sdxl} & - & EXIF as text & 48.47	& 82.91 & 0.1398 & 17.88 & 	18.17 \\
& 2 & SDXL~\cite{podell2023sdxl} & - & DoF as text & 53.80 & 81.93 &	0.0563 & 17.87 & 18.17 \\
& 3 & 4-step SDXL (Distilled) & - & DoF as text  & 52.85 & 79.67	& 0.0829 & 14.19 & 18.06 \\
& 4 & 4-step SDXL (Distilled) & - & EXIF embedding & 53.76 & 88.30 & 0.0344 & 15.38 & 16.92 \\
& 5 & Camera Settings as Tokens~\cite{fang2024}  & - & EXIF embedding & 56.23 & 66.97 & 0.2311 & 28.07 & 28.55 \\ %
\midrule
\parbox[t]{1mm}{\multirow{5}{*}{\rotatebox[origin=c]{90}{Gen. + Blur}}}
& 6 & SDXL~\cite{podell2023sdxl} & TAF~\cite{wang2023neural}  & EXIF as text & 67.04 &	79.54 & 0.1409 & 26.02 &	30.23 \\
& 7 & SDXL~\cite{podell2023sdxl} & Dr.Bokeh~\cite{sheng2024dr}  & EXIF as text & 79.40 &	81.07 & 0.1506 & 18.04 &	19.5 \\
& 8 & SDXL (EXIF-Fixed)~\cite{podell2023sdxl} & Dr.Bokeh~\cite{sheng2024dr}  & EXIF as text & 81.10 &	87.04 & 0.0356 & 19.01 &  18.8 \\
& 9 & Deep-DoF Gen & TAF~\cite{wang2023neural}  & Aperture, Focal Length &  82.12 &	87.62 & 0.0338 & 18.54 &	24.04 \\
\noalign{\vskip 0.3ex}
\cdashline{3-10}\noalign{\vskip 0.7ex}
& 10 &Ours & TAF~\cite{wang2023neural} & Aperture, Focal Length & 93.91 & {\bf 92.34} & {\bf 0.0064} & {\bf 13.24} & {\bf 16.69} \\
& 11 &Ours & Dr.Bokeh~\cite{sheng2024dr} & Aperture, Focal Length & {\bf 96.89} & 91.42 & 0.0144 & 13.67 & 17.51 \\

\bottomrule
\label{tab:results}
\end{tabularx}
}
\vspace{-6mm}

\end{table*}

\begin{table*}[ht]
\small
\centering 

\caption{{\bf Ablations.} Every quantitative metric becomes worse as components are removed from our full model (Ours + TAF~\cite{wang2023neural}).} %
\vspace{-2.0ex}
\begin{tabular}{c c cccc ccccc} 
\toprule
\multirow{2}{*}{\textbf{\#}} & \textbf{DoF}  & \textbf{Deep DoF}  & \textbf{Lens}  & \textbf{Focus Distance}  & \textbf{Blur} $\uparrow$  & \textbf{Content} $\uparrow$  & \multirow{2}{*}{\textbf{LPIPS} $\downarrow$} & \multirow{2}{*}{$\textrm{\bf FID}_\textrm{\bf DDoF} \downarrow$} & \multirow{2}{*}{$\textrm{\bf FID}_\textrm{\bf SDoF}\downarrow$} \\
& \textbf{datasets} & \textbf{Pretraining} & \textbf{Model} & \textbf{Transformer} & \textbf{Monotonicity}  & \textbf{Consistency} & & \\

\midrule

1 & \cmark & & & & 56.34 & 86.45 & 0.1432 & 15.67 & 16.95\\ %
2 & \cmark & \cmark & & & 57.51 & 88.51 & 0.0862 & 14.23 & 17.31 \\ %
3 & \cmark & \cmark & \cmark & & 73.21 & 90.65 & 0.0138 & 14.10 & 17.54 \\ %
\midrule
4 & \cmark & \cmark & \cmark & \cmark & 93.91 & 92.34 & 0.0064 & 13.24 & 16.69 \\
\cdashline{1-10}\noalign{\vskip 0.7ex}
5 & SDoF only & \cmark & \cmark & \cmark & 82.50 & 90.04 & 0.0031 & 14.05 & 16.99 \\

\bottomrule
\label{tab:ablations}
\end{tabular}

\vspace{-8mm}

\end{table*}

\subsection{Evaluation}
To assess the efficacy of our model, we evaluate two key properties: (1) how the blur amount changes in response to EXIF information, particularly aperture value, and (2) whether scene content remains unchanged as blur is adjusted. %
The lens models TAF~\cite{wang2023neural} and Dr.Bokeh~\cite{sheng2024dr} that we use have been previously validated in isolation for paired data and metrics. We define these metrics:

\vspace{-2mm}
\xhdr{Blur Monotonicity.} To verify that the model appropriately reflects changes in aperture, we check if decreasing the aperture value leads to increased defocus in the generated images. This is tested by evaluating whether the signal energy decreases as the aperture value decreases. Formally, for $\text{ap} \in \{\text{ap}_i\}_{i=1}^{N}$ with $\text{ap}_{i} < \text{ap}_{i+1}$ and $I_{\text{ap}_i}$ representing the image generated for each aperture value, we count the percentage of instances where  $E(I_{\text{ap}_{i}}) < E(I_{\text{ap}_{i+1}})$. Here, $E(\cdot)$ denotes the signal energy, which can be computed either as the sum of squared magnitudes of the 2D Fourier spectrum as $\sum_{\vec{k}} \left| \text{FFT2}(\cdot)_{\vec{k}} \right|^2$, where the sum is over all frequencies $\vec{k}$, or, by Parseval's theorem, equivalently as the sum of squared magnitudes in the spatial domain, $\sum_{p} \left|(\cdot)_{p} \right|^2$ over pixels $p$. We include in Supplemental Sec.~\ref{sec:evaluation-metrics} 
a statistical analysis on real photos and a proof (for a simple case of a scene with uniform depth) that this metric decreases from all-in-focus images to defocused images. A higher value for this metric indicates better model performance in controlling blur based on aperture. We compute this metric given the 8 aperture values: [1.8, 2.8, 4, 5.6, 8, 11, 16, 22].

\begin{figure*}[ht]
    \centering
    \includegraphics[width=\linewidth]{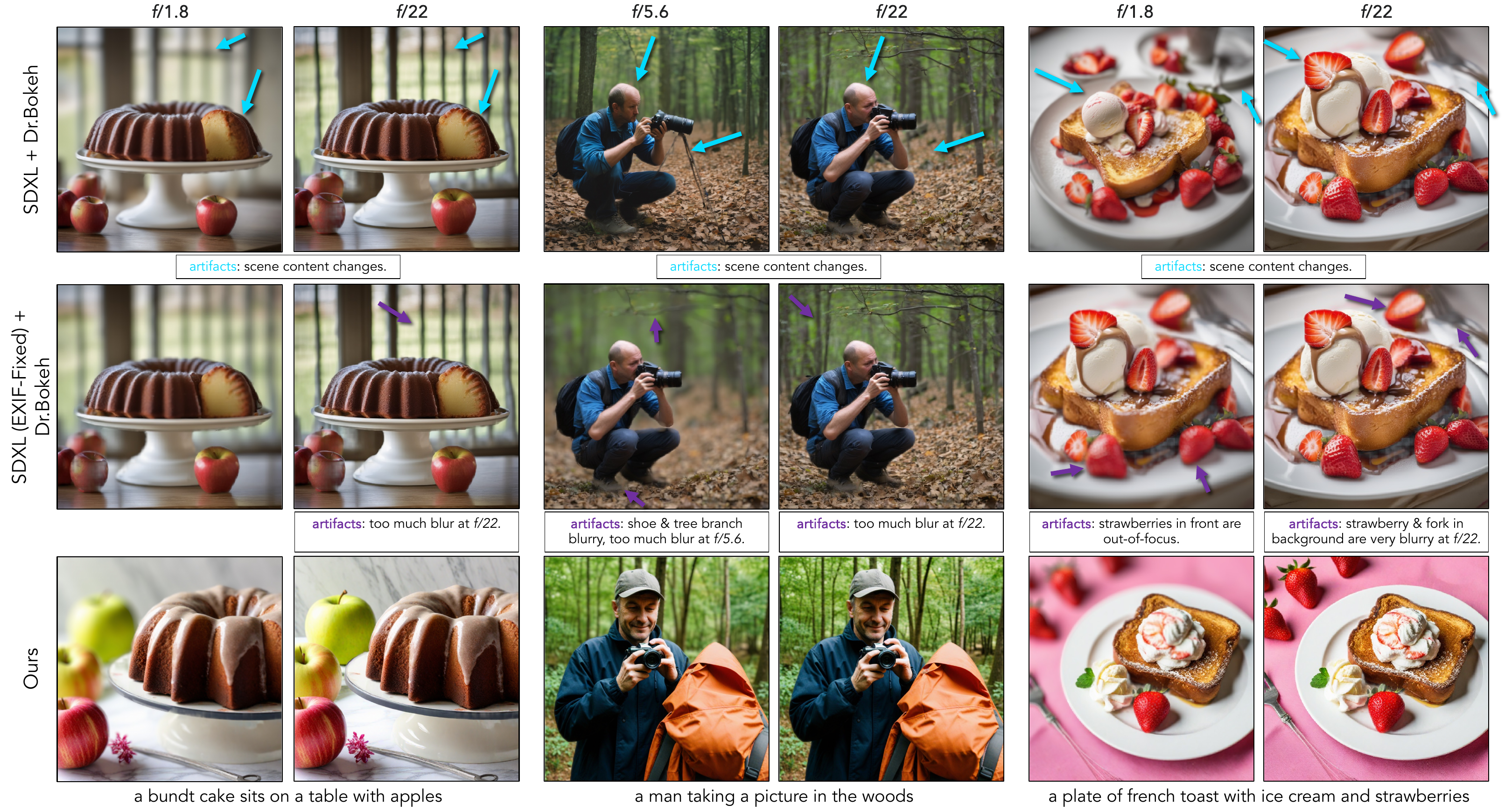}
    \vspace{-7mm}%
    \caption{{\bf Qualitative Comparisons with SDXL + Dr.Bokeh.} SDXL + Dr.Bokeh struggles to preserve scene content across aperture settings, with unrealistic defocus even when EXIF input to the generator is fixed. In contrast, our method effectively reduces blur while maintaining scene structure as aperture increases.  At $f/22$, we produce sharp images as expected, while SDXL + Dr.Bokeh introduces noticeable background blur. \blue{Blue arrows}: undesired content changes, \purple{purple arrows}: unrealistic defocus effects.
    } 
    \label{fig:qual_comparison}
    \vspace{-4mm}
\end{figure*}

\begin{figure*}[t]
    \centering
    \begin{center}
    \hspace*{-2mm} 
\includegraphics[trim={0 0 0 0},clip,width=1.02\textwidth]{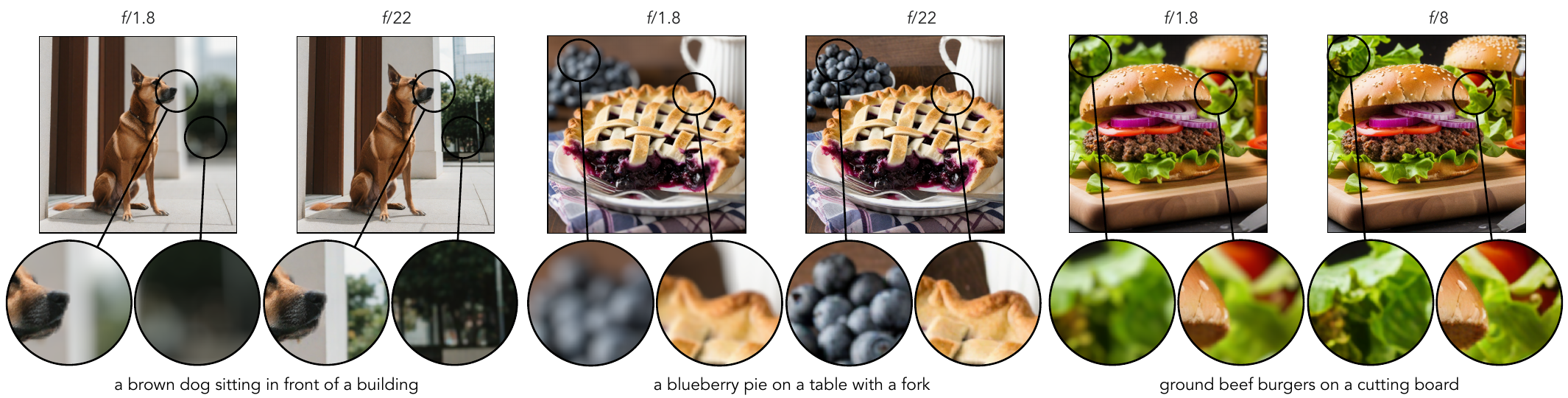}
    \end{center}
    \vspace{-6mm}%
    \caption{{\bf Qualitative Results.} We show our model’s performance across different apertures and prompts. As the aperture increases, background blur decreases while the overall scene remains consistent. {\em Please see the website videos for more results.}} 
    \label{fig:qual_results}
    \vspace{-5mm}
\end{figure*}

\vspace{-2mm}
\xhdr{Content Consistency.} This metric assesses whether the scene content remains unchanged as the amount of defocus varies in the image. %
For a set of aperture values [4, 5.6, 8, 11, 16, 22], we generate an image at each aperture, compute its semantic segmentation~\cite{chen2023semantic}, and compare each pixel’s segmentation class across images generated at different apertures. If a pixel’s class remains the same, we count its contribution as 1; otherwise, as 0. We calculate the mean across all samples. Higher values for this metric indicate a stronger ability to separate defocus effects from scene-content changes.

\vspace{-2mm}
\xhdr{LPIPS.} Our content consistency metric uses semantic segmentation, which has a benefit of being less sensitive to changes in blur, but might not capture certain fine-grained textural changes. Thus, we also evaluate LPIPS~\cite{zhang2018unreasonable}. We report the mean of the $\mathrm{LPIPS}(I_{\text{ap}_{i}}, I_{\text{ap}_{i+1}})$ metrics between adjacent apertures across $i=0, \ldots, 4$ for the same aperture values $\text{ap}_{i}$ as in content consistency.

\vspace{-2mm}
\xhdr{FID.} We compute the Fréchet Inception Distance (FID)~\cite{heusel2017gans} scores on 10,000 samples from both the Deep DoF and Shallow DoF datasets to assess image quality. A low {$\textrm{\bf FID}_\textrm{\bf DDoF}$ indicates that the model generates sharp, blur-free images, while a low {$\textrm{\bf FID}_\textrm{\bf SDoF}$ suggests that the model effectively reproduces images with regional blur.%

\section{Results}

\vspace{-2mm}
\xhdr{Baselines.} We categorize baselines into two groups: (1) methods that function purely as image generators and (2) methods that integrate a lens blur model with the generator (``Gen + Blur"), which are inspired by our framework but use more pretrained models rather than fully end-to-end training as we do. We evaluate our method against the following baselines.

\vspace{-2mm}
\xhdr{SDXL}: We use Stable Diffusion XL to evaluate this dataset, running it for 50 denoising steps. Since it does not natively support EXIF encoding, we test two methods for incorporating EXIF data. In the first variant, we convert Aperture and Focal Length to strings and add them as additional text input. In the second variant, we use a Depth-of-Field (DoF) prompt based on the aperture value: for apertures greater than 10, we set the prompt to ``Deep Depth-of-Field”, and for apertures less than 10, we set it to ``Shallow Depth-of-Field'' and add it to the text.

\vspace{-2mm}
\xhdr{4-step SDXL (Distilled)}: We distill the SDXL teacher model into a 4-step generator and train it with DMD2 losses. We train two variants: one that includes a DoF prompt in the text input and another that uses an EXIF projection layer for improved EXIF data processing where we add sinusoidal positional encoding for Aperture and Focal Length and process them using two projection layers and add it to the time embedding. See the supplemental section~\ref{sec:4-step-sdxl} for details regarding the EXIF projection.

\xhdr{Camera Settings as Tokens}: We evaluate ~\cite{fang2024} as a baseline, where camera parameter embeddings and a LoRA are trained to guide image generation based on EXIF metadata. The model takes aperture, focal length, ISO rating, and exposure time as inputs in addition to the text prompt and uses Stable Diffusion 2 ~\cite{Rombach_2022_CVPR} as the base generator.

\xhdr{SDXL + TAF Lens}: This baseline is inspired by our framework, but utilizes pretrained models in a plug-and-play manner. We use SDXL as the image generator and add the TAF lens model~\cite{wang2023neural} to simulate defocus effects. Depth is predicted using the Metric3Dv2 model \cite{hu2024metric3d}. To determine an optimal focus distance, we use a frozen Visual Saliency Transformer~\cite{Liu_2021_ICCV} model to generate a saliency mask, and we use the mean depth value of the salient region as the focal distance prediction.

\xhdr{SDXL + Dr.Bokeh}: This baseline is like SDXL + TAF Lens, except we replace the lens model with Dr.Bokeh~\cite{sheng2024dr}.

\xhdr{SDXL (EXIF-Fixed) + Dr.Bokeh}: As in SDXL + Dr.Bokeh, but the EXIF metadata fed to SDXL is fixed across aperture settings to preserve scene content, closely following the strategy of our framework.

\xhdr{Deep-DoF Gen. + TAF Lens}: Inspired by our framework, we train a model similar to ours, but instead of training a model to predict focus distance, we again use the mean depath of the salient region from the frozen VST model~\cite{Liu_2021_ICCV}.

\begin{figure*}[t]
    \centering
    \includegraphics[width=\linewidth]{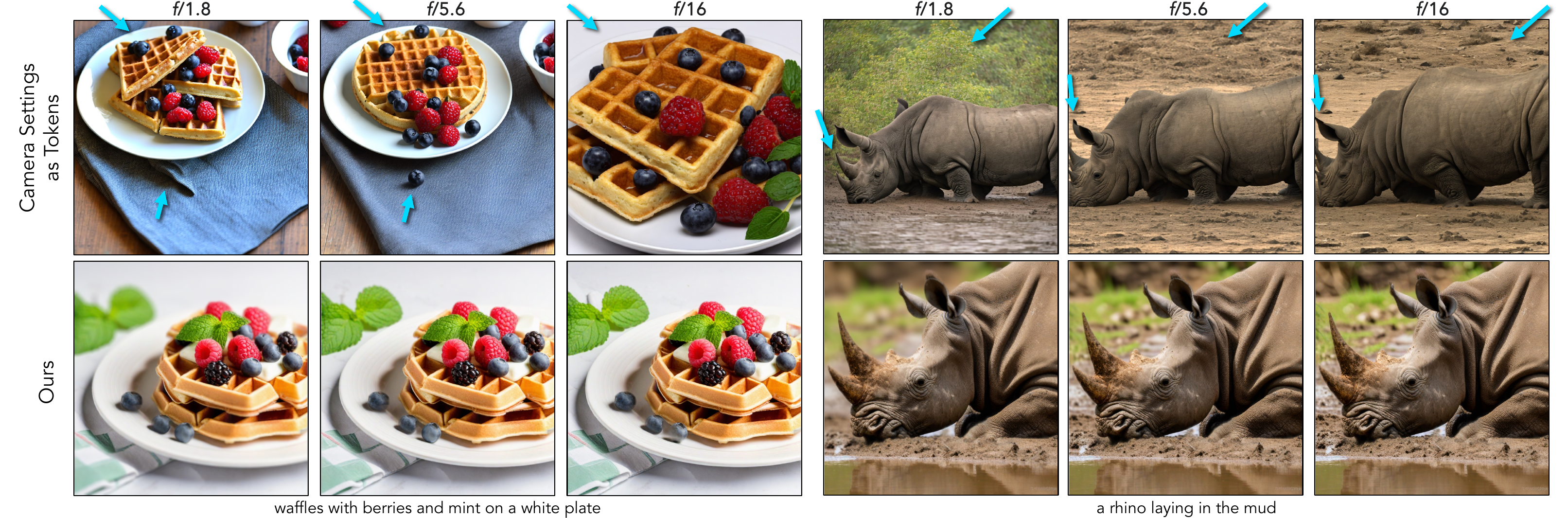}\vspace{-3mm}%
    \caption{{\bf Comparison to Camera Settings as Tokens~\cite{fang2024}}. Despite being trained with EXIF conditioning, ~\cite{fang2024} struggles to control blur effectively and often alters the image significantly, making it entirely different. Whereas, our model preserves scene content while successfully decoupling defocus blur from the scene, modifying only the blur without changing the overall composition. \blue{Blue arrows}: undesired content changes.}
    \vspace{-3mm}
    \label{fig:camera-tokens-comparison}
\end{figure*}

\begin{figure*}[h]
    \centering
    \includegraphics[width=\linewidth]{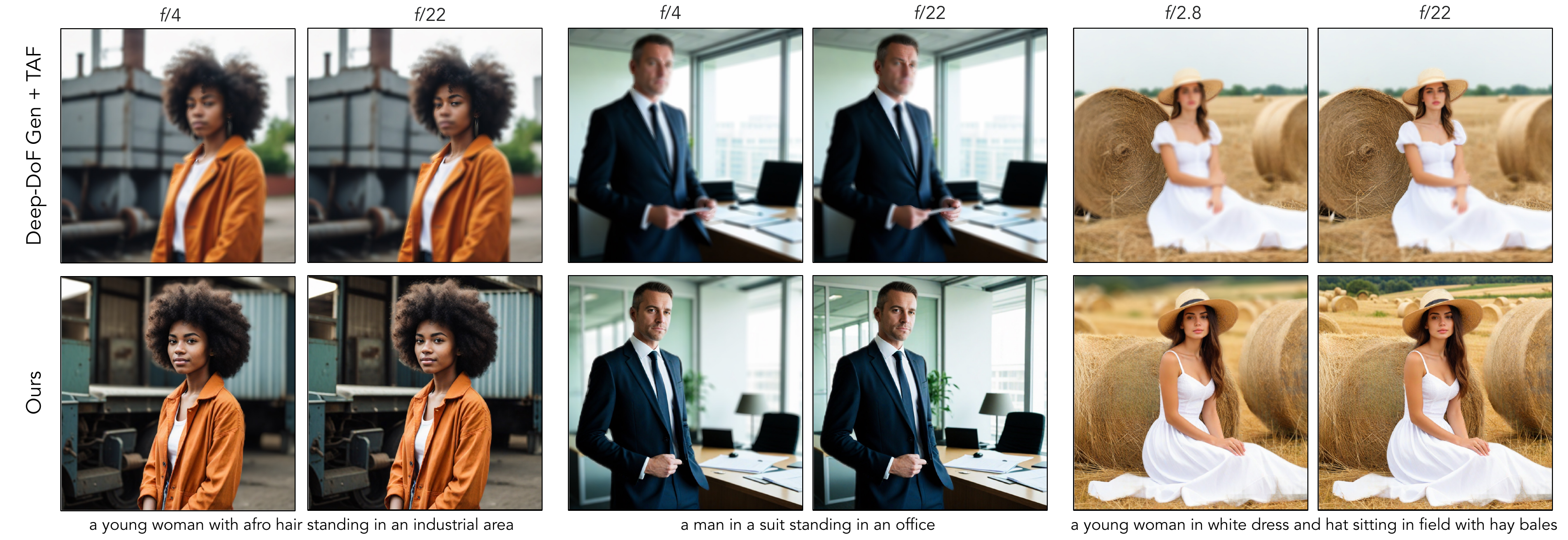}
    \vspace{-7mm}
    \caption{{\bf Comparison to Deep-DoF Gen. + TAF.} We compare our method with Deep-DoF Gen + TAF model (main paper Table~\ref{tab:results}, row 8). Our method is able to consistently vary blur amount while preserving the scene, whereas the generated images from the baseline appear out-of-focus without focusing on salient regions in the image.}
    \label{fig:fig11}
    \vspace{-5mm}
\end{figure*}

\xhdr{Quantitative Results.} We evaluate all baselines in Table~\ref{tab:results}. The SDXL baselines without lens models, both with and without distillation, and with and without EXIF conditioning (rows 1-4) all struggle to generate different defocus effects as aperture and focal distance are changed. %
Camera Settings as Tokens (row 5), despite being explicitly conditioned on EXIF metadata, struggles to adjust blur effectively based on aperture changes and struggles to maintain scene integrity. This is seen in its low Blur Monotonicity and Content Consistency scores (see also Figure~\ref{fig:camera-tokens-comparison}).

The SDXL + TAF Lens baseline (row 6), a fully pretrained model, outperforms earlier methods by leveraging additional information such as depth and focus distance, enabling more consistent blur application as the aperture decreases. Replacing the lens model with Dr.Bokeh (row 7) further improves blur rendering. Further fixing the EXIF input to SDXL (row 8) preserves the scene content (following our framework), but often generates out-of-focus blur (Fig.~\ref{fig:qual_comparison}). %
The baseline in row 9 resembles our approach, except it does not have focus distance prediction and only trains on the DDoF dataset. Without learning an appropriate scale for focus distance (as our model does), it struggles, often producing out-of-focus blur effects, and also obtains a poor FID score on the SDoF dataset. Our model with TAF Lens (row 10) outperforms all baselines, consistently increasing blur as the aperture decreases while preserving scene content across aperture changes. Replacing TAF Lens with Dr.Bokeh during inference (row 11) further improves the blur metric, due to Dr.Bokeh’s superior blur rendering capabilities.

\xhdr{Ablations.}
We ablate key components of our model in Table~\ref{tab:ablations}.  All ablations are trained on our Shallow and Deep DoF datasets (except row 5). 
Rows 1 and 2 exclude the lens model and instead use aperture and focal length as additional text inputs to control blur behavior. These configurations show reduced Blur Consistency, indicating that representing aperture solely as text does not provide sufficient information to consistently simulate blur changes. Notably, row 2 outperforms row 1, demonstrating the benefit of pretraining on the Deep DoF dataset. This pretraining provides our generator $\mathbf{G}$ with a strong prior for generating all-in-focus images during fine-tuning.

In row 3, the addition of the lens model improves Content Consistency when changing aperture and enhances the model’s ability to apply blur consistently at a given aperture. Row 3 removes the focal distance model and replaces it with mean depth, which compared to the full model gives worse visual quality (FID) and worse blur monotonicity as the focus distance scale is not learned, causing unintended out-of-focus blur and limiting the lens model’s ability to increase blur beyond a certain point. Finally, in row 4, adding the focus distance model allows the lens model to better focus on salient regions, reducing out-of-focus artifacts. Row 5 shows pretraining the generator on Deep-DoF and freezing it, followed by finetuning only on Shallow DoF, leads to worse all-in-focus image generation.

\xhdr{Qualitative Results.} We show qualitative results for our method and baselines. 
In Figure~\ref{fig:qual_comparison}, we compare our method, using Dr.Bokeh as the lens model against an SDXL generator with a Dr.Bokeh lens model (Table~\ref{tab:results}, row 7). Our method effectively preserves scene content while adjusting the blur effect based on the specified aperture value, enabling precise control over defocus blur during image generation. In contrast, SDXL + Dr.Bokeh struggles to increase blur as the aperture changes and alters the scene content. Even when we fix the EXIF input to SDXL, the defocus blur is not realistic.  
Figure~\ref{fig:qual_results} illustrates the blurring effects of our method at different aperture values. It keeps the foreground (or salient region) in focus while modifying the blur in non-salient regions without altering the overall scene.

In Figure~\ref{fig:camera-tokens-comparison}, we compare against the Camera Settings as Tokens~\cite{fang2024} baseline. This method fails to maintain scene content, significantly altering the generated images while providing limited control over blur. Figure~\ref{fig:fig11} further compares our method with the Deep-DoF Gen + TAF baseline, which struggles to produce aesthetically pleasing focal planes, because it neither learns an appropriate focus-distance scale nor is fine-tuned on a shallow DoF dataset. In Supplemental Section~\ref{sec:comp-to-our-model}, we extend the analysis: (1) evaluating a ControlNet-based variant conditioned on scene depth to improve the Camera Settings as Tokens approach, and (2) adding additional comparisons with real photographs. Even with conditional depth, camera-setting embeddings, and a text prompt, the ControlNet variant still fails to preserve scene content consistently.

\section{Limitations, Future Work, Conclusion}

Due to strong priors in the SDXL model that we distill from, the generated images from the all-in-focus generator $G$, on rare occasions, have a background blur already present, which limits how in focus of an image can be obtained. This could likely be mitigated by further increasing the DDoF dataset size. The focus distance scale is learned from weak signals from DMD2 losses and does not have a direct supervisory signal, which can in some cases result in blurry photos without a good focal plane. Our unsupervised method has the advantage of not needing explicit collection of focus distance values, which are typically not present or accurate in EXIF tags. However, future work might consider the  task of capturing RGBD photos with metric depth along with precise measurement of focus distances to enable supervised learning. Even with the high-quality Dr.Bokeh renderer, high-frequency details along occlusion boundaries such as the dog fur in Figure~\ref{fig:qual_results} can occasionally be blurry; this is due to limited depth-map resolution and could be resolved by using a depth estimator that resolves such high-frequency details \cite{Miangoleh_2021_CVPR}. We learn only disc-shaped bokehs, but a benefit of our framework is that the lens model can be swapped out at inference time (Sec.~\ref{sec:lens_model}), so in future work, other lens models or stylized bokehs such as hexagons or hearts could be used~\cite{yang2016virtual}.

In conclusion, our approach is a flexible framework that decouples the image generator from a lens model. This approach is useful for obtaining explicit and fine-grained control of aperture and focus properties in diffusion generative models. Our framework shows the benefits of explicit focus distance estimation and our joint end-to-end learning on both DDoF and SDoF datasets. %

{
    \small
    \bibliographystyle{wacv_template/ieeenat_fullname}
    \bibliography{main}
}

\clearpage
\appendix

\maketitlesupplementary

\setcounter{figure}{0}
\renewcommand{\thefigure}{A\arabic{figure}}
\setcounter{table}{0}
\renewcommand{\thetable}{A\arabic{table}}

\setcounter{section}{0} 

\begin{figure*}[!t]
    \centering
    \includegraphics[width=\linewidth]{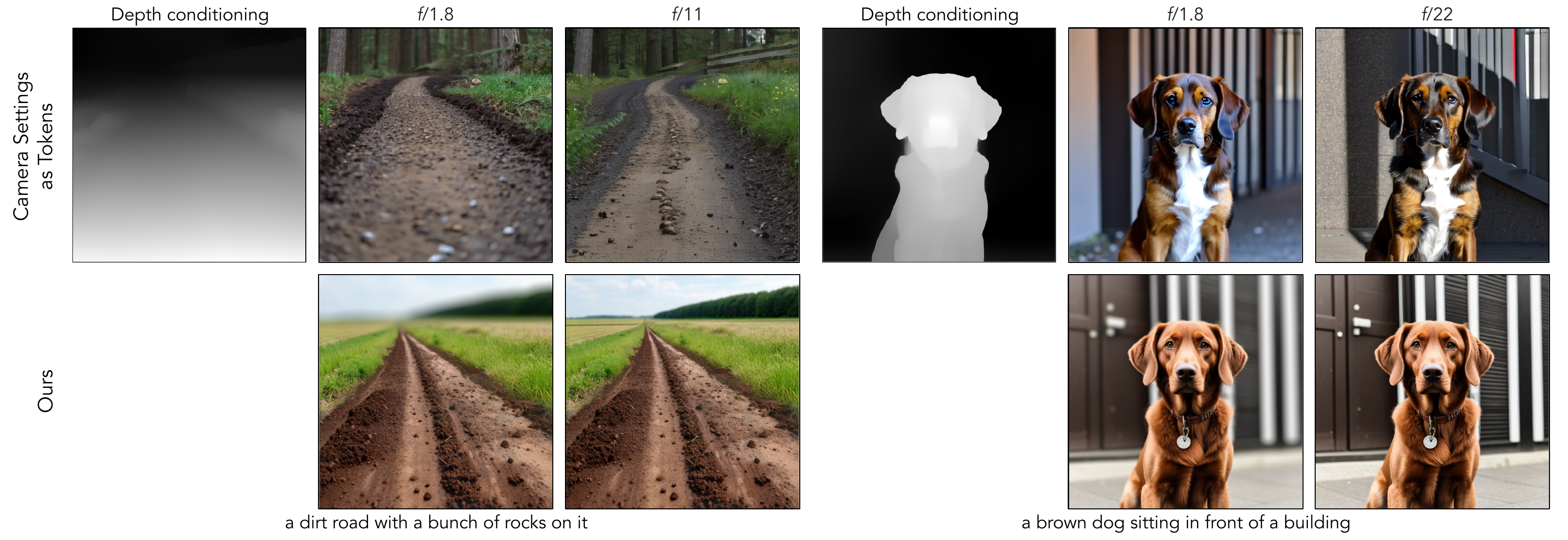}%
    \caption{{\bf Comparison with Camera Settings as Tokens~\cite{fang2024} based ControlNet~\cite{zhang2023adding}}. We compare our method to a depth-conditioned ControlNet that uses Camera Settings as Tokens embeddings.  While the ControlNet effectively adheres to scene depth, it alters scene content within those depth planes. Notably, depth is used as a conditioning input for ControlNet but not for our generator.}
    \label{fig:camera-tokens-controlnet-comparison}
    \vspace{2mm}
\end{figure*}

\begin{figure*}[h]
    \centering
    \includegraphics[width=\linewidth]{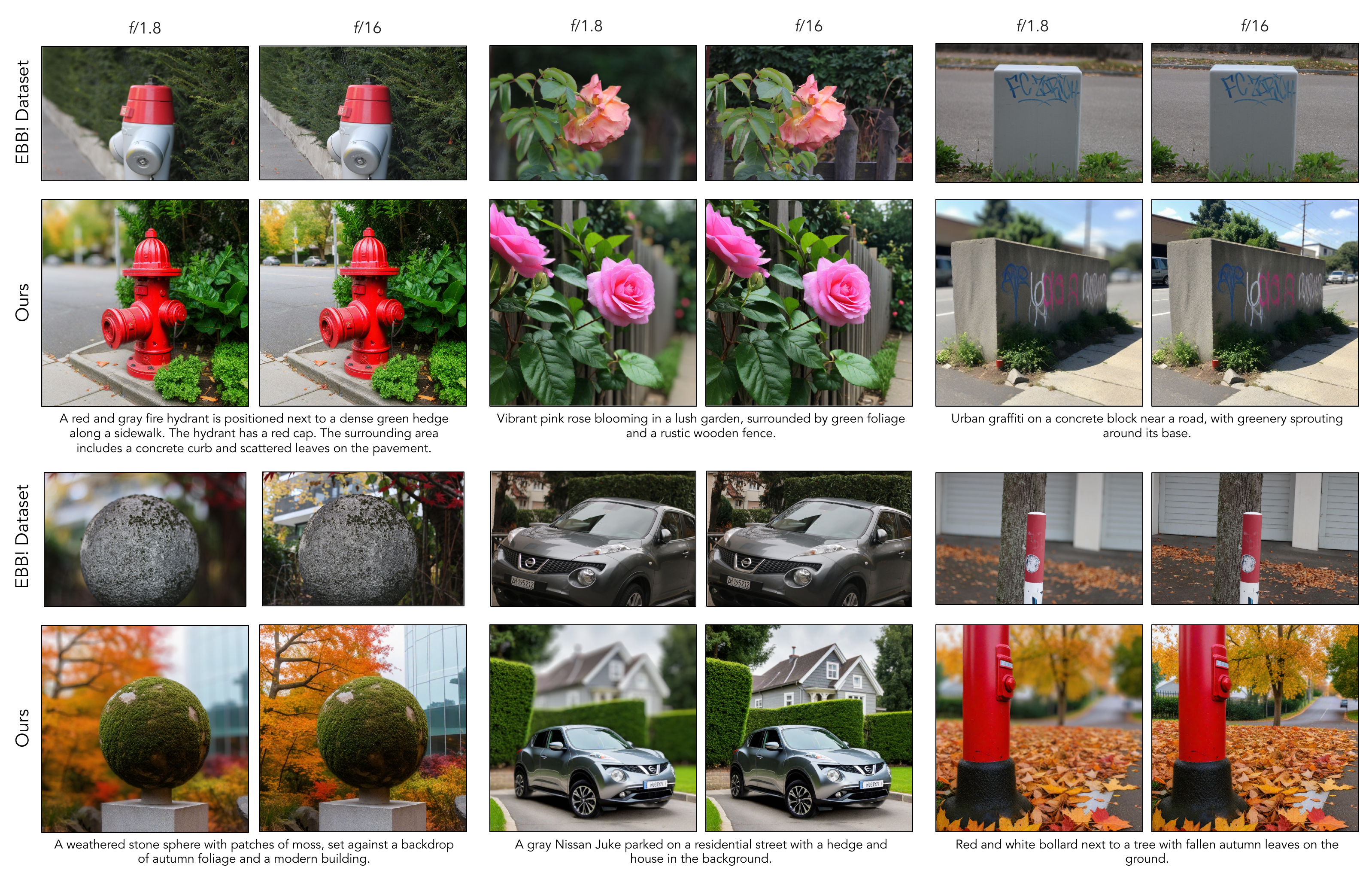}
    \vspace{-6mm}
    \caption{{\bf Comparison to EBB! Dataset.}  
    Using the EBB! dataset~\cite{ignatov2020rendering}, which provides image pairs captured at two apertures: $f/16$ (all-in-focus) and $f/1.8$ (shallow depth of field), we first caption the $f/16$ image with the InternVL3~\cite{zhu2025internvl3} model (shown below the ``Ours'' images).  
    We then use that caption as the text prompt, along with the specified aperture ($f/16$ or $f/1.8$), to generate images with our method.  
    Our results closely match the expected defocus characteristics, producing pronounced blur at $f/1.8$ and sharp, well-focused images at $f/16$.}
    \label{fig:fig12}
\end{figure*}

\begin{figure*}[h]
    \centering
    \includegraphics[width=\linewidth]{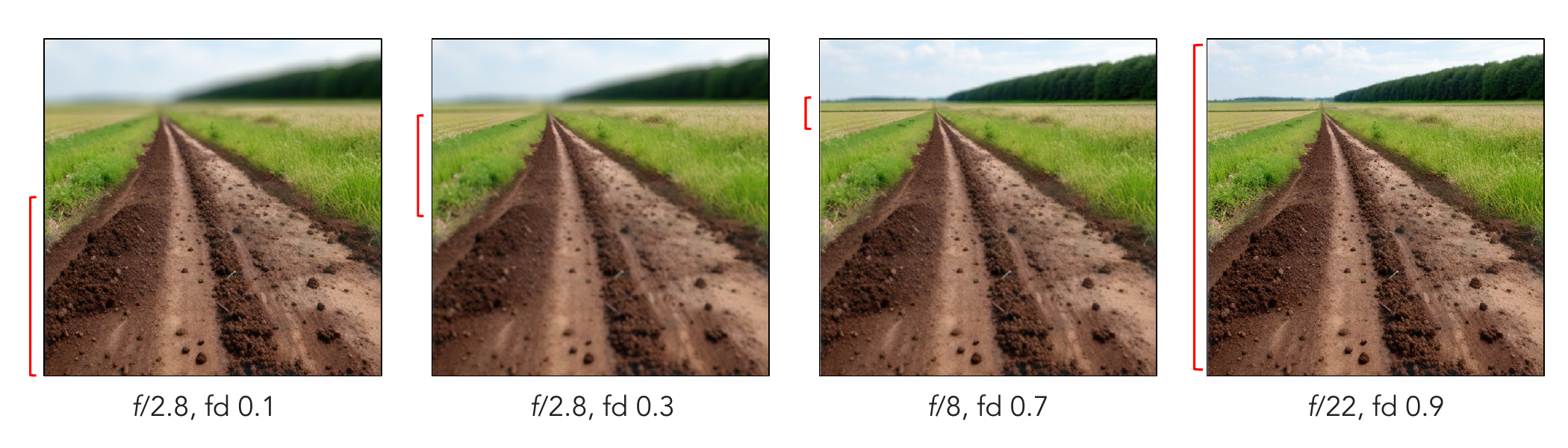}%
    \vspace{-2mm}
    \caption{{\bf Varying focus distances in the generation process.} We show that varying the focus distance in our model produces images with focus shifting across different focal planes. As the focus distance increases from low to high values, the focal plane transitions from the near to the far plane. The red bar \red{[} highlights the region of the image that is in focus.}
    \label{fig:focus_distance_change}
\end{figure*}

\section{Video}
We provide videos on our webpage ({\normalsize \texttt{\href{https://ayshrv.com/defocus-blur-gen}{link}}}) showing the controllability of defocus blur using our model. We also include qualitative examples demonstrating defocus control in generated images for several prompts.

\section{Comparisons to our model}
\label{sec:comp-to-our-model}

\xhdr{Comparison with the ControlNet baseline.}
Figure~\ref{fig:camera-tokens-controlnet-comparison} examines an alternative approach where a ControlNet~\cite{zhang2023adding} is conditioned on depth to improve scene preservation for Camera Settings as Tokens. Despite using conditional depth, camera embeddings, and a text prompt, this approach still struggles to maintain scene content. As seen in Figure~\ref{fig:camera-tokens-controlnet-comparison} (right), the baseline preserves the dog's pose but changes its identity and the background scene. In contrast, our method maintains both the subject and background while effectively adjusting the blur.

\xhdr{Qualitative Comparison with Real Images.}
To assess our model’s qualitative performance against real photographs, we compare its outputs with the Everything is Better with Bokeh! (EBB!) dataset~\cite{ignatov2020rendering}.  
This dataset contains pairs of images of the same scene captured at two aperture settings: $f/1.8$ (shallow depth-of-field) and $f/16$ (all-in-focus).  

To generate comparable results without bias toward either aperture, we first caption each $f/16$ image using the InternVL3 model~\cite{zhu2025internvl3}.  
These captions serve as neutral text prompts.  
We then provide the caption along with the target aperture ($f/1.8$ or $f/16$) to our model to synthesize shallow depth-of-field and all-in-focus images, respectively.  
Representative outputs are shown in Figure~\ref{fig:fig12}.

The figure demonstrates that our method faithfully reproduces the expected optical characteristics of each aperture.  
When conditioned on $f/1.8$, our model produces images with pronounced background blur and smooth bokeh, closely matching the shallow-focus ground truth and showing sharp foreground with naturally defocused backgrounds.  
When conditioned on $f/16$, it generates images with crisp details across the full depth of field, consistent with the all-in-focus reference photographs.  

These results confirm that our approach not only captures the semantic content of a scene but also accurately models the physical effects of aperture on defocus, validating the effectiveness of our aperture-aware image generation framework.

\section{Controllability of defocus blur in image generation}

Our model takes EXIF metadata (e.g., aperture, focal length) and a text prompt as input, generating an image that faithfully reflects both. A trained focus distance model predicts the scene’s focus distance during generation, which the lens model uses to apply defocus blur consistent with the metadata.

To enable controllability over the focus in the generated image, users can intercept the predicted focus distance and provide their own focus distance value. The lens model applies spatial blur based on this user-defined focus distance, allowing precise control over where the generated image should focus. The focus distance is represented on the depth output scale of the Metric3Dv2 depth model. For instance, a focus distance of 0.1 corresponds to the depth plane with a value of 0.1 in the depth map. We demonstrate the effects of varying focus distance in Figure~\ref{fig:focus_distance_change}, where low and high focus distance values result in noticeable shifts in the focal plane within the image.

In addition to the prompt, our model provides the ability to manipulate focus distance and aperture, offering fine-grained control over image generation. By leveraging this information, the model determines where and how much defocus blur to apply. This controllability is illustrated in a video attached in the supplementary material, along with several qualitative examples.

\section{Deep and Shallow Depth-of-Field Datasets}
\label{sec:suppl_datasets}
\begin{figure}[t]
    \centering
    \includegraphics[width=3.2in]{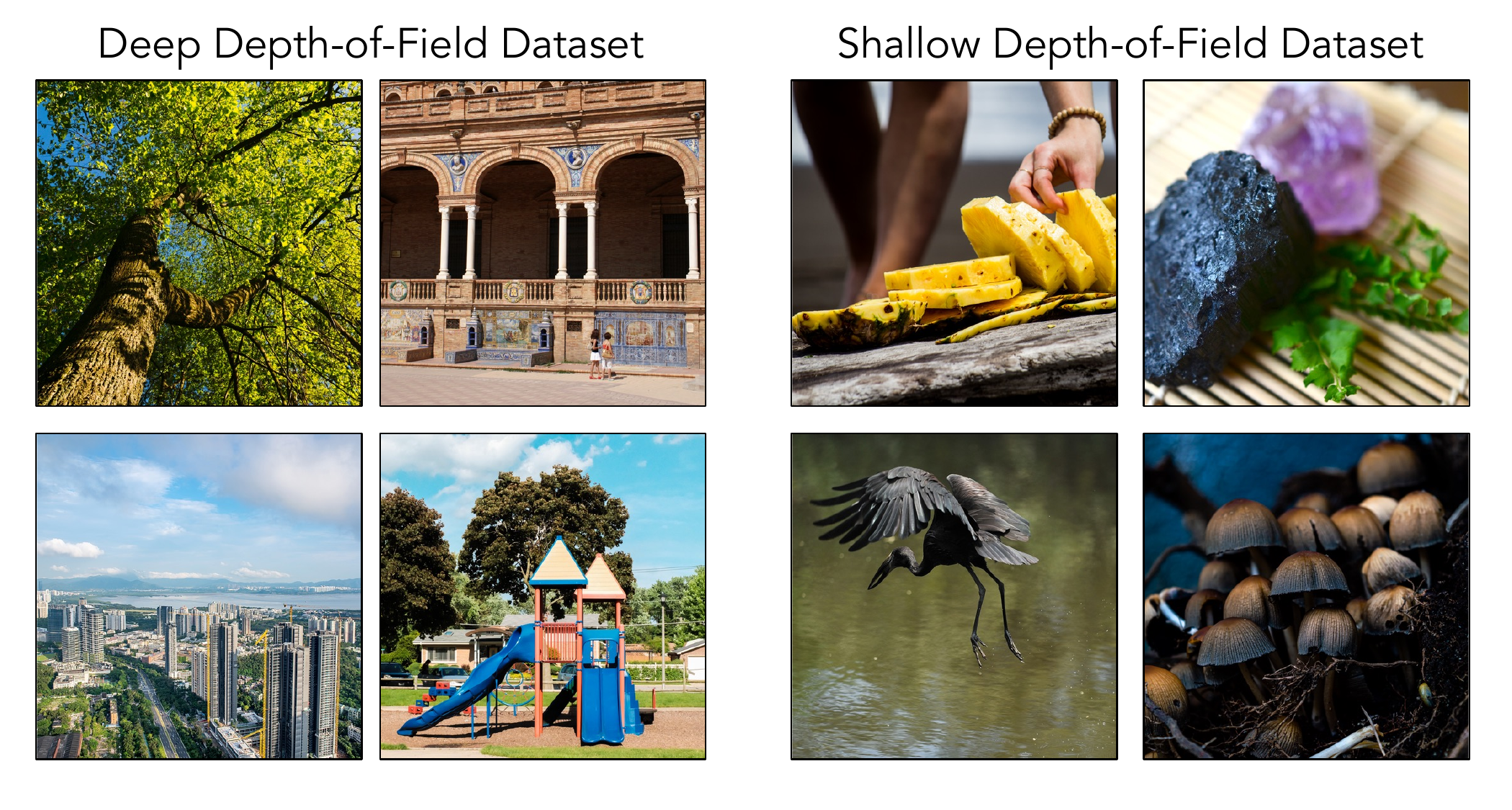}\vspace{-4mm}%
    \caption{{\bf Deep and Shallow DoF Datasets.} Images shown are selected using our dataset filtering approach mentioned in Sec.~\ref{sec:suppl_datasets}. After filtering, the Deep DoF dataset primarily consists of all-in-focus images, while the Shallow DoF dataset includes images with defocus blur, emphasizing a specific object of interest. We use these datasets to train our model.}
    \label{fig:ddof_sdof_dataset}
    \vspace{-2mm}
\end{figure}

Our generator $G$ produces all-in-focus (deep depth-of-field, or
Deep DoF) images $\mathbf{x}$, while the lens model renders shallow DoF
images $\hat{\mathbf{x}}$, trained with only weak supervision ($\mathbf{x}$, $\hat{\mathbf{x}}$ shown in Fig.~2).
To supervise this pipeline, we curate large-scale datasets of deep and shallow DoF images from roughly 300 million uncurated photographs drawn from a commercially available stock-photography dataset. We discard photos with no EXIF data. For photos without captions, we generate captions using BLIP2~\cite{li2023blip}. A ResNeXt–FPN classifier~\cite{Xie2016} is used to identify whether each image contains no blur, desirable blur, or undesirable blur. 

\paragraph{Filtering Criteria.}
From the classifier outputs and EXIF metadata, we apply the following
filters to construct the two datasets:
\begin{itemize}[label=$\bullet$,itemsep=2pt,topsep=5pt,leftmargin=5mm]
\item{\textbf{Aperture range:} Shallow DoF images retain aperture values below 10, producing a narrow depth of field and aesthetically pleasing background blur. Deep DoF images retain aperture values between 10 and 50 to ensure sharp focus across the scene.}

\item{\textbf{Device type:} Smartphone photographs are removed from the shallow DoF set to avoid synthetic blur introduced by computational photography.}

\item{\textbf{Exposure time:} Images with exposure times longer than 0.1 seconds are excluded from the deep DoF set to prevent motion blur.}

\item{\textbf{Photographic validity:} We discard non-photographic content (e.g., AI-generated or illustrated images) by verifying that each candidate is a real photograph using the vision–language model InternVL~\cite{chen2023internvl}.}
  
\item{\textbf{Blur classifier output:} Shallow DoF images are those labeled as exhibiting the desired blur, while deep DoF images are those labeled as having no blur.}

\end{itemize}

\paragraph{Dataset Scale.}
Applying these criteria gives us roughly 1.5 million
(image, EXIF, prompt) pairs for each of the Deep DoF and Shallow DoF
datasets. Representative samples are shown in Figure~\ref{fig:ddof_sdof_dataset}.

\begin{figure*}[t]
    \centering
    \includegraphics[width=6.9in]{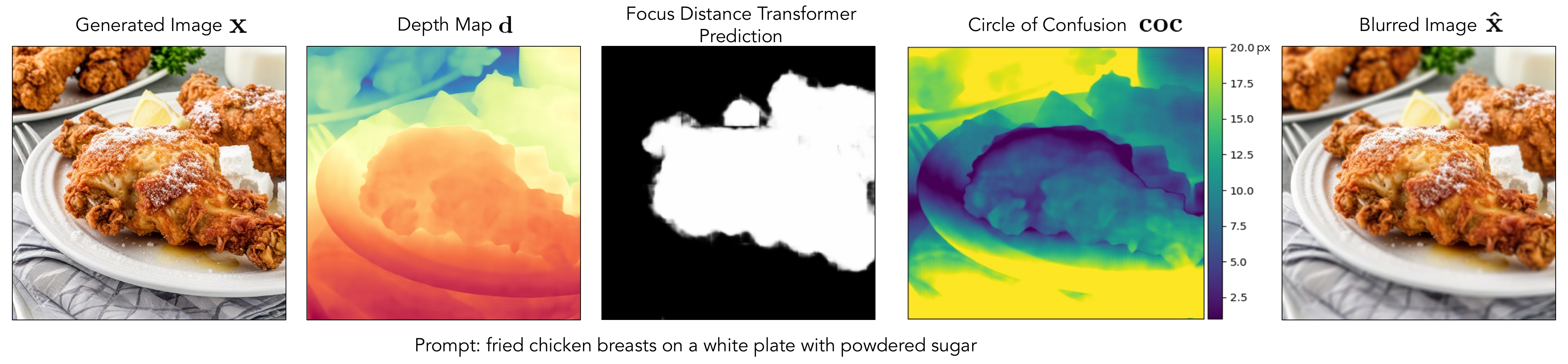}\vspace{-2mm}%
    \caption{{\bf Image Generation Pipeline.} The pipeline begins with an image generated by the model (left), followed by depth prediction from the depth model. A saliency map is then predicted and used to compute the focus distance as a weighted sum of depth and saliency. The lens model calculates the circle of confusion (CoC) based on depth, focus distance, and other EXIF parameters. Finally, a spatially varying blur kernel, derived from the CoC, is applied to the generated image. The entire pipeline is trained end-to-end.}
    \label{fig:image_pipeline}
    \vspace{-2mm}
\end{figure*}

\section{Human Study}
\label{sec:human_study}

We conducted a human study to validate our method, aiming to evaluate whether the model can preserve the scene and reduce defocus blur when the aperture value in the camera metadata increases. For the study, we used 25 prompts from the validation split of the deep and shallow depth-of-field datasets we created. For each prompt, we generated images corresponding to aperture values in the set [1.8, 2.8, 4, 5.6, 8, 11, 16, 22] across all methods.

The study involved six baseline methods in addition to our approach:
\begin{itemize}[label=$\bullet$,itemsep=2pt,topsep=5pt,leftmargin=5mm]
\item{SDXL (Table 2, Row 2),}
\item{4-step SDXL (Distilled) (Table 2, Row 4),}
\item{Camera Settings as Tokens (Table 2, Row 5),}
\item{SDXL (EXIF-Fixed) + Dr.Bokeh Lens (Table 2, Row 8),}
\item{SDXL + TAF Lens (Table 2, Row 6),}
\item{Deep-DoF Gen + TAF Lens (Table 2, Row 8).}
\end{itemize}

We created a video for each method per prompt, sequentially increasing the aperture value to illustrate its effect in the video. During the study, participants were shown paired videos—one generated by our model and the other by a baseline—for the same prompt. Participants were instructed to select the video that better preserved scene content, reduced blur as aperture increased, and kept the salient object in focus.

Each participant answered 20 comparison questions, with video pairs randomly assigned. The study involved 25 participants, and their aggregated preferences are presented in Figure~\ref{fig:human_study}. Results show that users preferred our method over the baselines in over at least 83\% of cases, consistent with the performance metrics in Table 2, further demonstrating our method’s superiority over baselines.

The user study was conducted on the Hugging Face Spaces~\cite{huggingfacepspace} platform, and the interface used is shown in Figure~\ref{fig:human_study_interface}.

\section{Implementation Details}
Here, we provide more information about training, hyperparameters, and evaluation metrics.

\subsection{Training and Hyperparameters}
We train the few-step generator by distilling it from the SDXL model~\cite{podell2023sdxl} over 4 steps.
We train our network on 2 nodes, each equipped with 8 A100 GPUs, for a total of 16 GPUs. 
The Deep DoF generator is trained for one day, followed by an additional day of fine-tuning using the full setup, which includes the lens model, depth estimation module, and focus distance predictor. To scale the training efficiently across 8 nodes, we use the Fully-Shared Data Parallel framework~\cite{zhao2023pytorch}. 

The training images have a resolution of 1024 $\times$ 1024, and the model is optimized using the AdamW optimizer with a learning rate of $5 \times 10^{-7}$, a weight decay of 0.01, and beta parameters of (0.9, 0.999). The batch size is set to 1 to fit the entire model in GPU memory. The fake diffusion model $\mu_{\text{fake}}$ is updated 5 times for each generator update and during generator updates, we alternate between Shallow and Deep DoF. The focus distance model (optimized with $L_{\text{Huber}}$) is updated at every iteration. The guidance scale for the real diffusion model $\mu_{\text{real}}$ is to be 8. The loss weights are set to  $\lambda_1=1, \lambda_2=1, \lambda_3=200$.

\subsection{Evaluation Metrics}
\label{sec:evaluation-metrics}
\xhdr{Content Consistency.} To evaluate this metric, we compute the segmentation maps using Semantic Segment Anything~\cite{chen2023semantic}. This is an open-set segmentation method which means it does not have a predefined set of prediction classes. Due to this, sometimes the top-1 predicted class could be different for the same object. So, we compare the top-3 predicted classes. To check if the semantic class remains the same, we check if any of top-3 classes matches remain the same for the image pixels instead of comparing just top-1.

\xhdr{Blur Monotonicity.} We introduced Blur Monotonicity to quantify whether image blur decreases as the aperture value increases.   To measure the efficacy of this metric, we use the Everything is Better with Bokeh! dataset~\cite{ignatov2020rendering},
which provides approximately 5{,}000 pairs of all-in-focus images captured at $f/16$ and corresponding shallow-depth images at $f/1.8$. In 96\% of the pairs, the signal energy of the all-in-focus image exceeds that of its bokeh counterpart, supporting the premise of our metric.  Visual inspection of the remaining 4\% reveals negligible defocus differences, making the energy comparison less informative in those specific cases.

We now present a theoretical justification for the validity of our metric. As a reminder, we defined the signal energy $E(\cdot)$ as the sum of squared magnitudes of the 2D Fourier spectrum, computed as $\sum_{\vec{k}} \left| \text{FFT2}(\cdot)_{\vec{k}} \right|^2$, where the sum is over all frequencies $\vec{k}$. For a simple scene of uniform depth (i.e., depth is constant across pixels), we show that the energy of an image formed from that scene is greater than the energy of the image after blurring via convolution with a blur kernel.

\begin{theorem}
Let $f, h$ be $d$-dimensional tensors with $f, h \in \mathbb{R}^{N_1 \times N_2 \times \cdots \times N_d}$, where $\vec{N} = (N_1, N_2, \ldots, N_d)$. Define the discrete Fourier transform (DFT) of $f$ as

\begin{equation}
F_{\vec{k}} = \sum_{\vec{n}=\vec{0}}^{\vec{N}-1} f_{\vec{n}} \, e^{-2\pi i \, \vec{k} \cdot \left( \frac{\vec{n}}{\vec{N}}\right) },
\end{equation}

where division is element-wise. Define $H_{\vec{k}}$ analogously for $h$. The multi-index summation is defined as
\[
\sum_{\vec{n}=\vec{0}}^{\vec{N}-1} := \sum_{n_1=0}^{N_1 - 1} \sum_{n_2=0}^{N_2 - 1} \cdots \sum_{n_d=0}^{N_d - 1} = \sum_{\vec{n} \in \{0, \ldots, N_1 - 1\} \times \cdots \times \{0, \ldots, N_d - 1\}}.
\]

Assume $h_{\vec{n}}\ge 0$ for all $\vec{n}\in\{0,\ldots,\vec{N}-1\}$ and
$\sum_{\vec{n}=\vec{0}}^{\vec{N}-1} h_{\vec{n}} = 1$. 
If $g = f \ast h$ and $G_{\vec{k}}$ is the DFT of $g$, then:

\begin{equation}
\label{eqn:energy_decreases}
\sum_{\vec{k}=\vec{0}}^{\vec{N}-1} \left| G_{\vec{k}} \right|^{2} = \sum_{\vec{k}=\vec{0}}^{\vec{N}-1} \left| F_{\vec{k}} \, H_{\vec{k}} \right|^{2} \leq \sum_{\vec{k}=\vec{0}}^{\vec{N}-1} \left| F_{\vec{k}} \right|^{2}.
\end{equation}

\end{theorem}

\begin{proof}
The equality in the above equation follows from the convolution theorem, which states that $G_{\vec{k}} = F_{\vec{k}} \, H_{\vec{k}}$. Now we can analyze the inequality. For each frequency $\vec{k}$,
\begin{align}
H_{\vec{k}} &= \sum_{\vec{n}=\vec{0}}^{\vec{N}-1} h_{\vec{n}} \, e^{-2\pi i \, \vec{k} \cdot \left( \frac{\vec{n}}{\vec{N}}\right) }
\\
\left| H_{\vec{k}} \right| &\leq \sum_{\vec{n}=\vec{0}}^{\vec{N}-1} \left| h_{\vec{n}} \right| \left| e^{-2\pi i \, \vec{k} \cdot \left( \frac{\vec{n}}{\vec{N}}\right)} \right| = \sum_{\vec{n}=\vec{0}}^{\vec{N}-1} \left| h_{\vec{n}} \right| = 1
\end{align}
Since $\left|e^{i\theta}\right|=1$ for all $\theta\in\mathbb{R}$. We want to show that 
\begin{equation}
0 \geq \sum_{\vec{k}=\vec{0}}^{\vec{N}-1} \left| F_{\vec{k}} \right|^{2} \left( \left| H_{\vec{k}} \right|^{2} - 1 \right).
\end{equation}
Now $\forall \vec{k}, \left| F_{\vec{k}} \right|^{2} \geq 0$ and $\left| H_{\vec{k}} \right|^{2} - 1 \leq 0$, so 
\begin{equation}
\sum_{\vec{k}=\vec{0}}^{\vec{N}-1} \left| F_{\vec{k}} \right|^{2} \left( \left| H_{\vec{k}} \right|^{2} - 1 \right) \leq 0.
\end{equation}
Further, the inequality in Eq.~\ref{eqn:energy_decreases} is strict if there exists some $\vec{k} \; \text{such that} \left| F_{\vec{k}} \right| > 0 \; \text{and} \left| H_{\vec{k}} \right| < 1.$
\end{proof}

\begin{application}
Assume the Circle of Confusion (CoC) is not spatially varying (i.e., the scene has uniform depth) and let $f$ be the image and $h$ the blur kernel. 
By the theorem, the signal energy satisfies
\[
E(h \ast f)
    \le E(f).
\]

Further, the inequality is strict if there exists some $\vec{k} \; \text{such that} \left| F_{\vec{k}} \right| > 0 \; \text{and} \left| H_{\vec{k}} \right| < 1$. Assume the image is formed from the scene by a process that adds i.i.d.\ Gaussian noise to each pixel. Then, almost surely, $\left| F_{\vec{k}} \right| > 0$ for all $\vec{k}$, since the DFT is a linear transformation and each DFT coefficient is the sum of a deterministic component and a Gaussian-distributed random variable. Thus, it suffices to analyze the kernel $h$ and determine whether there exists some $\vec{k}$ with $\left| H_{\vec{k}} \right| < 1$. In particular, any blur kernel $h$ that is non-negative, sums to 1, and is not a delta function (i.e. has at least two non-zero entries) guarantees a strict inequality. This includes, as special cases, disc-shaped and polygonal bokeh corresponding to blur kernels with such shapes. We emphasize again that since this analysis uses the convolution theorem, it applies only to the simplified setting of a scene with uniform depth.

\end{application}

Although the theoretical guarantee holds under the assumption of uniform depth and spatially invariant blur, our empirical results on real-world images with spatially varying blur strongly suggest that the blur monotonicity metric remains a reliable indicator of relative blur. This combination of theory and empirical validation supports the practical utility of our metric.

\subsection{4-step SDXL (distilled) with EXIF}
\label{sec:4-step-sdxl}
To incorporate camera metadata into the distilled 4-step SDXL generator, we design an EXIF projection module that encodes numerical tags --- specifically, Aperture and Focal Length. These values are first transformed using sinusoidal positional embeddings, then concatenated and passed through two projection layers to produce a single EXIF embedding, which is added to the diffusion timestep embedding.

\begin{figure*}[htbp] 
    \centering
    \vspace{5pt}
    \includegraphics[width=0.9\linewidth]{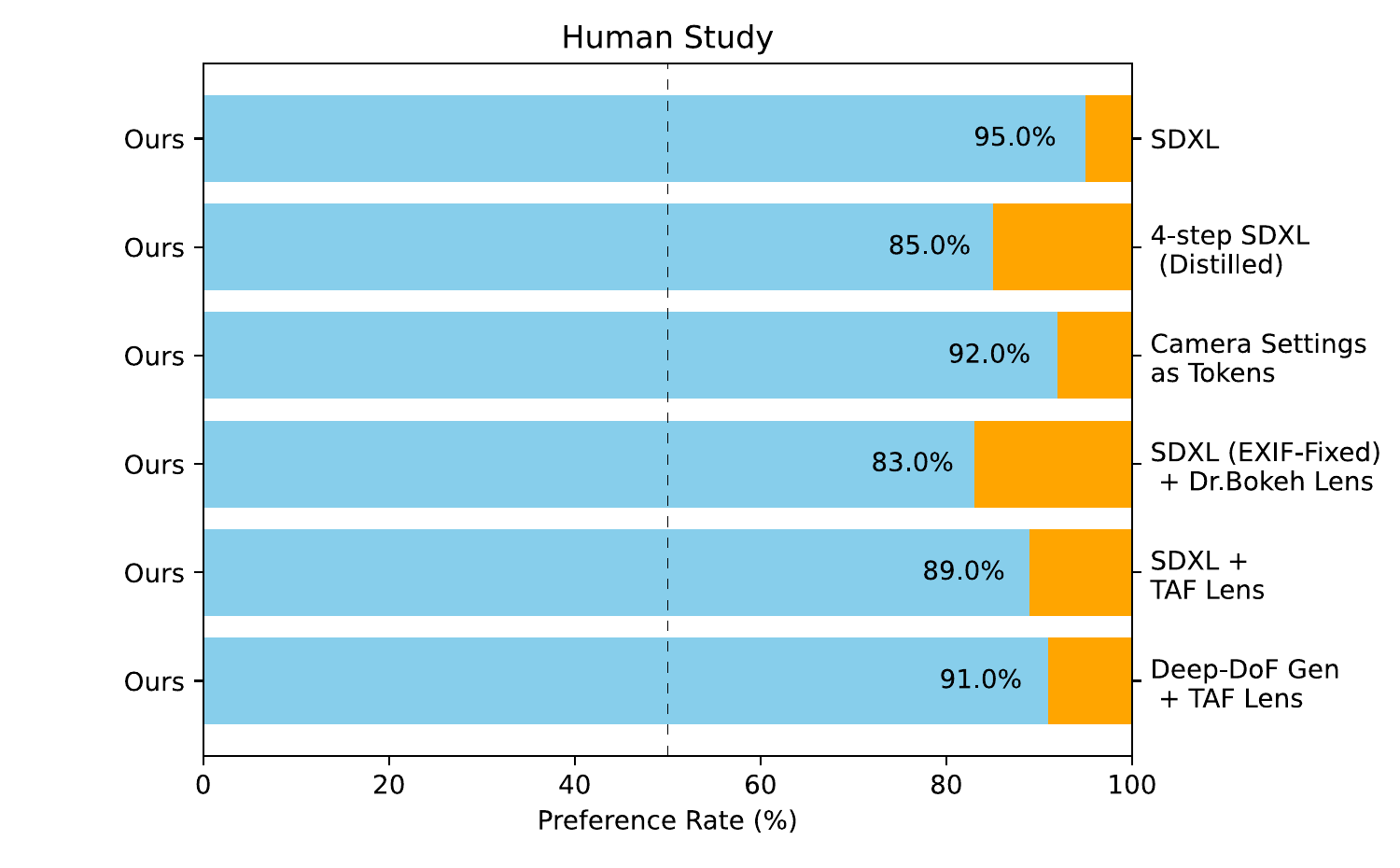}%
    \vspace{5pt}
    \caption{{\bf Human Studies.} Preference rate for selecting our method over the baselines (Section ~\ref{sec:human_study}).}
    \label{fig:human_study}
\end{figure*}

\begin{figure*}[t]
    \centering
    \includegraphics[width=0.9\linewidth]
    {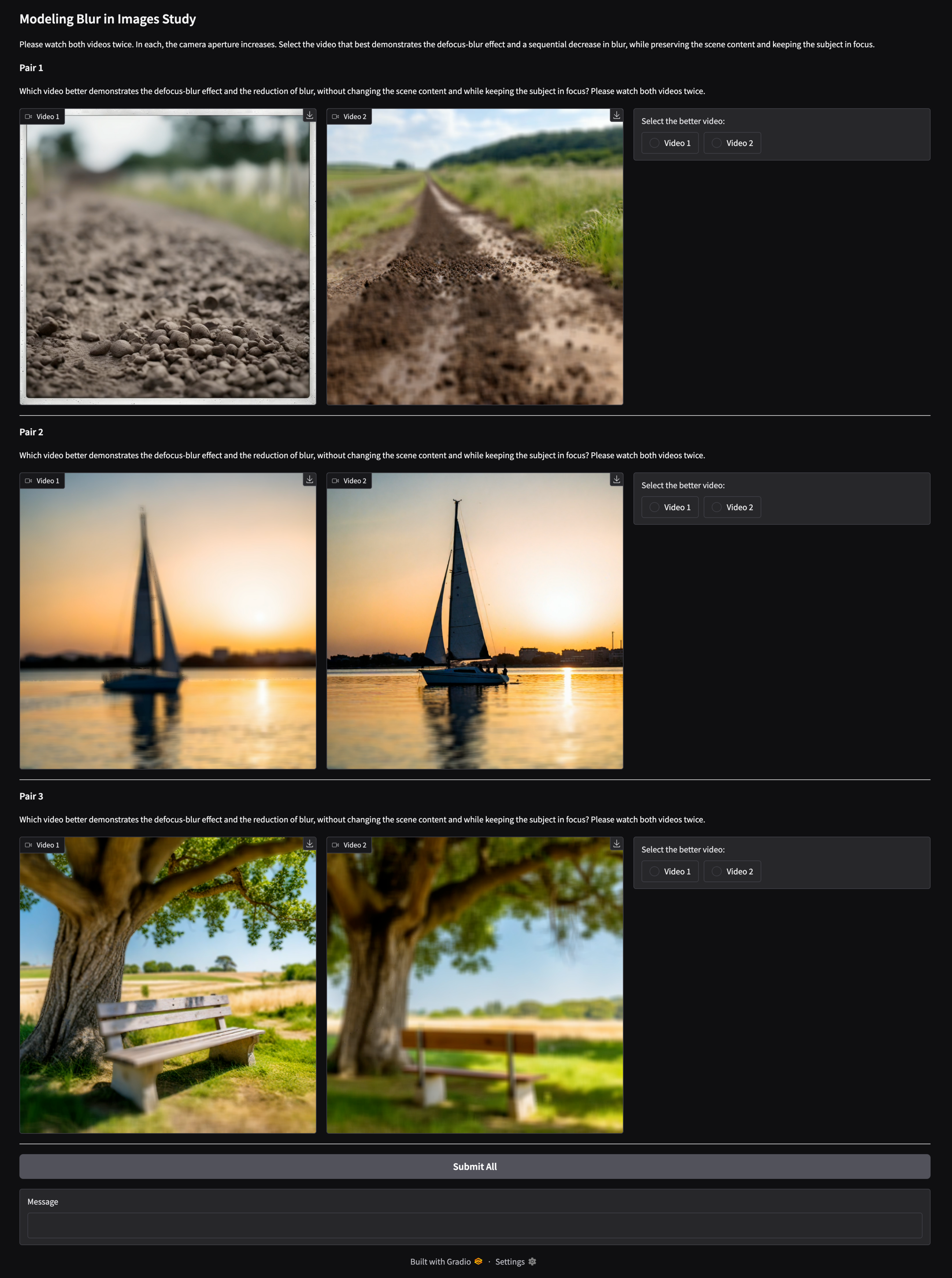}\
    \vspace{-2mm}
    \caption{{\bf Human study interface.} We show the Hugging Face Spaces interface used to conduct the user studies. In each question, one video is generated by our method, while the other is randomly selected from one of the baselines for the same prompt. The participants are tasked with selecting the video that shows the realistic defocus-effect and the least amount of scene content change with decrease in blur as the aperture increases in the video. }
    \label{fig:human_study_interface}
    \vspace{-5mm}
\end{figure*}

\end{document}